\documentclass[11pt]{article}

\usepackage{amsfonts,amssymb,amsmath,amsthm,dsfont,bm}
\usepackage{graphicx}
\usepackage[linesnumbered,ruled,lined,boxed,commentsnumbered]{algorithm2e} 
\usepackage{mathtools} 
\usepackage{subfigure}
\usepackage{hyperref}
\usepackage{url}
\usepackage[dvipsnames]{xcolor}
\usepackage[footnotesize]{caption}
\usepackage{cite}
\usepackage[margin=2cm]{geometry}

\title{Quantized Compressive K-Means}
\author{Vincent Schellekens\thanks{E-mail: {\em \{vincent.schellekens, ~laurent.jacques\}@uclouvain.be}. ISPGroup, ELEN/ICTEAM, UCLouvain (UCL), B1348 Louvain-la-Neuve, Belgium. VS and LJ are funded by Belgian National Science Foundation
(F.R.S.-FNRS).} \and Laurent Jacques\footnotemark[1]}

\newcommand{\Quant}{q}
\newcommand{\Rbb}{\mathbb{R}}
\newcommand{\scp}[2]{\langle #1, #2 \rangle}

\newtheorem{proposition}{Proposition}

\newcommand{\sign}{{\rm sign}\,}

\renewcommand{\leq}{\leqslant}
\renewcommand{\geq}{\geqslant}

\DeclareMathOperator*{\argmin}{arg\,min}

\newcommand{\bb}{\mathbb}

\newcommand{\ts}{\textstyle}
\newcommand{\bs}{\boldsymbol}
\newcommand{\cl}{\mathcal}
\newcommand{\ie}{\emph{i.e.}, }
\newcommand{\eg}{\emph{e.g.}, }

\renewcommand{\Vec}[1]{\bm{#1}} 
\newcommand{\distiid}{\overset{\mathrm{iid}}{\sim}} 
\newcommand{\expec}[1]{\mathop{{}\mathbb{E}}_{#1}} 
\newcommand{\im}{\mathrm{i}\mkern1mu} 
\DeclarePairedDelimiterX{\norm}[1]{\lVert}{\rVert}{#1} 
\newcommand{\Integr}[4]{\int_{#1}^{#2}#3\mathrm{d}#4} 
\DeclarePairedDelimiter\floor{\lfloor}{\rfloor}

\makeatletter
\newcommand{\RemoveAlgoNumber}{\renewcommand{\fnum@algocf}{\AlCapSty{\AlCapFnt\algorithmcfname}}}
\newcommand{\RevertAlgoNumber}{\algocf@resetfnum}
\makeatother

\begin{document}
\addtolength{\textfloatsep}{-15pt}
\RemoveAlgoNumber
\maketitle

\begin{abstract}
The recent framework of \textit{compressive statistical learning} proposes to design tractable learning algorithms that use only a heavily compressed representation---or \emph{sketch}---of massive datasets. Compressive K-Means (CKM) is such a method: it aims at estimating the centroids of data clusters from pooled, non-linear, random signatures of the learning examples. While this approach significantly reduces computational time on very large datasets, its digital implementation wastes acquisition resources because the learning examples are compressed only \textit{after} the sensing stage.

The present work generalizes the CKM sketching procedure to a large class of periodic nonlinearities including hardware-friendly implementations that \textit{compressively acquire entire datasets}. This idea is exemplified in a Quantized Compressive K-Means procedure, a variant of CKM that leverages 1-bit universal quantization (\ie retaining the least significant bit of a standard uniform quantizer) as the periodic sketch nonlinearity. Trading for this resource-efficient signature (standard in most acquisition schemes) has almost no impact on the clustering performance, as illustrated by numerical experiments.
\end{abstract}

\section{Introduction}
\label{sec:introduction}

Numerous scientific fields have recently experienced a paradigm shift towards data-driven approaches where mathematical models are inferred from a dataset of learning examples ${X  = \{ \Vec{x}_i  \in \Rbb^n\}_{i = 1}^N}$. 
\textit{K-means clustering} (KMC) \cite{Steinhaus1956} is such a method widely used in, \eg data compression, pattern recognition, and bioinformatics~\cite{Jain2010clustering50years,Steinley2006kmeanssynthesis}. Given $K$, a prescribed number of clusters (groups of similar data), KMC seeks the centroids (or ``cluster representatives") $\cl C  =  \{\Vec{c}_k  \in  \Rbb^n\}_{k = 1}^K$ minimizing the Sum of Squared Errors (SSE):
\begin{equation}
\cl C^* = \argmin_{\cl C} \mathrm{SSE}(\cl C) = \argmin_{\cl C} \: \sum_{i=1}^{N}\,\min_{1 \leq k \leq K} \: \|\Vec{x}_i - \Vec{c}_k\|^2. 
\label{eq:clustering}
\end{equation}
Solving~\eqref{eq:clustering} exactly is NP-hard \cite{Aloise2009clusterNPhard}, so in practice a tractable heuristic such as the popular \texttt{k-means} algorithm \cite{Lloyd1982kmeans, Arthur2007kmeans++} is widely used to find an approximate solution $\cl C_{\texttt{km}}$. However, \texttt{k-means} complexity scales poorly with the size of modern voluminous datasets where $N$ is typically $\cl O(10^3 - 10^6)$, or grows continually for \textit{data streams} processing. In fact, since \texttt{k-means} repeatedly requires---at each iteration---a thorough pass over $X$, this massive dataset must be stored and read several times, with prohibitive memory and time consumptions. Paradoxically, the large dataset size (\ie $nN$) dwarfs, and does not affect, the number of parameters learned by \texttt{k-means} (\ie $nK$). \emph{Ideally, larger datasets increase the model accuracy without requiring more training computational resources}.

\begin{figure}
	\centering
	\includegraphics[width=0.72\linewidth]{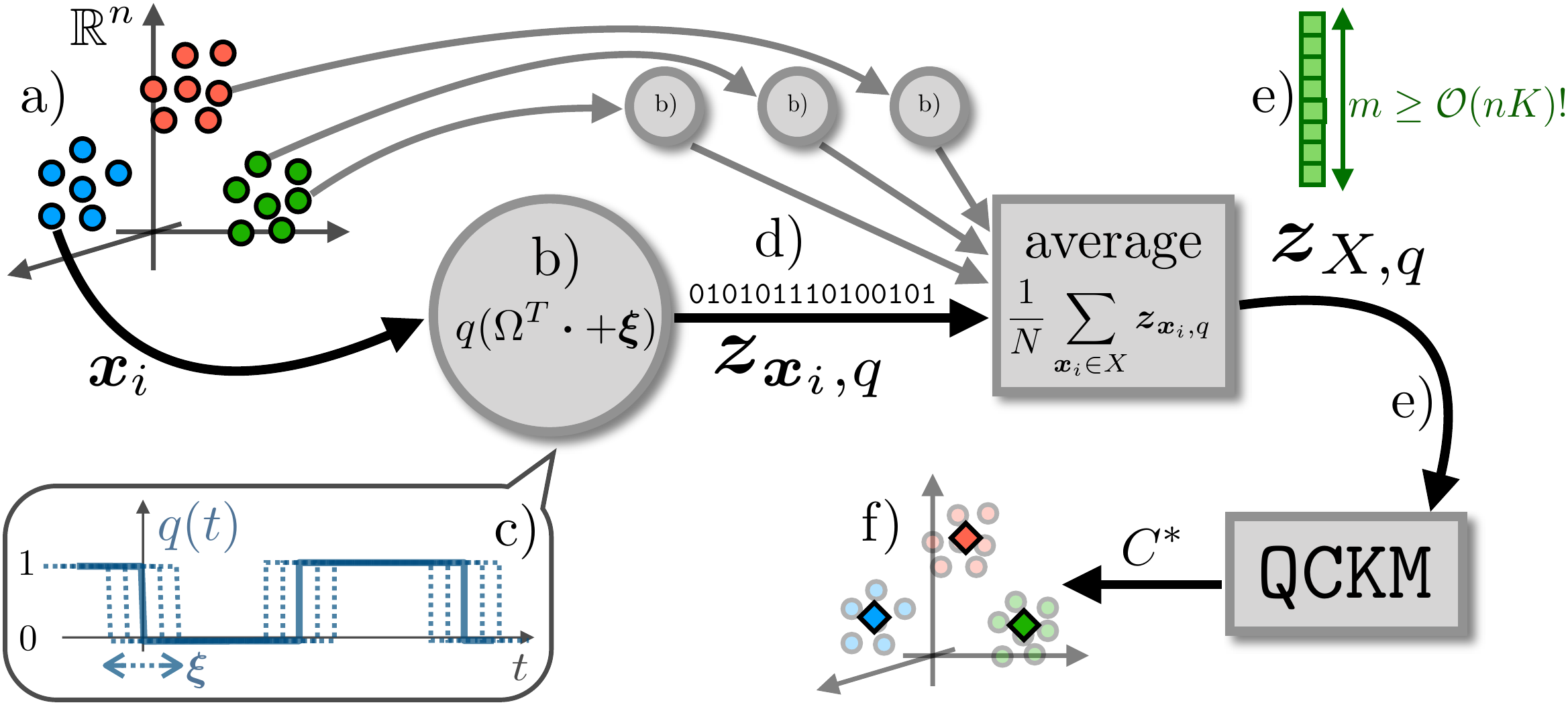}
	\caption{ \textbf{a)} A massive dataset $X$, composed of $K$ clusters of examples $\Vec{x}_i \in \Rbb^n$, is not explicitly available but acquired via one (or a cloud of) low-power sensor(s) \textbf{b)} implementing random projections on frequencies $\Omega$, a dithering $\Vec{\xi}$ and \textbf{c)} the 1-bit universal quantization $\Quant(\cdot)$ ($-1$ is encoded as $0$). Only \textbf{d)} the sketch contributions of the examples ($m$ bits) are acquired, and form after averaging \textbf{e)} the sketch, a highly compressed but meaningful representation of $X$. Our \texttt{QCKM} method then extracts \textbf{f)} the $K$ cluster centroids from it.}
	\label{fig:cloud}
\end{figure}

This goal motivates the recent \textit{compressive learning} framework \cite{gribonval2017compressiveStatisticalLearning}, where learning algorithms solely require access to a drastically compressed representation of the dataset called the \textit{sketch}, a single vector $\Vec{z}_X \in \mathbb{C}^m$, constructed by collecting $m$ (empirical) generalized moments of the dataset $X$:
\begin{equation}
\Vec{z}_{X} := \frac{1}{N} \sum_{i = 1}^N \Vec{z}_{\Vec{x}_i} \quad \text{where} \quad \Vec{z}_{\Vec{x}_i} := \left[\exp (- \im \Vec{\omega}_j^T\Vec{x}_i )\right]_{j=1}^{m},
\label{eq:basicSketch}
\end{equation}
with ``frequencies'' $\Vec{\omega}_j \in \Rbb^n$ sampled randomly according to a distribution $\Lambda$. The sketch is thus the \emph{pooling} (average) of random projections of the data samples after passing through a nonlinear, periodic signature---the complex exponential. The Compressive K-Means (\texttt{CKM}) method \cite{keriven2016compressive} clusters $X$ from $\Vec{z}_X$ by replacing~\eqref{eq:clustering} with a \textit{sketch matching} optimization problem:
\begin{equation}
(\cl C_{\texttt{CKM}},\Vec{\alpha}_{\texttt{CKM}}) = \argmin_{(\cl C,\Vec{\alpha})} \textstyle \|\Vec{z}_{X} - \sum_{k = 1}^K \alpha_k \Vec{z}_{\Vec{c}_k}\|^2,
\label{eq:CKMclustering}
\end{equation}
with the cluster weights $\alpha_k \geq 0$ satisfying $\sum_k \alpha_k = 1$. It was shown empirically that $\mathrm{SSE}(\cl C_{\texttt{CKM}}) \simeq \mathrm{SSE}(\cl C_{\texttt{km}})$ provided $m = \mathcal{O}(nK)$, \ie the required sketch size $m$ is only proportional to the number of parameters to learn, allowing for tractable memory consumption and training time \emph{whatever the number of training examples} $N$. It was also theoretically proven that $m = \cl O(nK^2)$ is a sufficient condition for retrieving meaningful centroids from the sketch \cite{gribonval2017compressiveStatisticalLearning}. Interestingly, the sketch is (up to a re-scaling) linear\footnote{A sketch $\Vec{\Phi}_S$ of a vector set $S$ is said \textit{linear} if $\Vec{\Phi}_{S \cup S'} = \Vec{\Phi}_{S} + \Vec{\Phi}_{S'}$.} and $\Vec{z}_X$ can thus be computed in one pass over the data, possibly realized in parallel over several machines. This sketch is also easy to update when new examples are available (\eg in data streams). \\

The following limitation in \texttt{CKM}s sketching strategy motivate this work: \emph{all signals $\Vec{x}_i \in X$ (or, equivalently, their projections onto the frequencies $\{\bs\omega_j\}_{j=1}^m$, as in a compressive sensing scheme~\cite{Candes2006CStheorem}) must be acquired and stored at full resolution (high bitrate) in order to expansively evaluate (in software) their contributions $\Vec{z}_{\Vec{x}_i}$ to the sketch}. A resource-preserving (\eg computational or energy efficient) compressive learning sensor should directly and solely acquire $\Vec{z}_{\Vec{x}_i}$. While random projections can be cheap to compute (\eg by using fast structured random projections \cite{chatalic2018fastSketch,Candes2008CSintro} or, possibly, by relying on optical random processes~\cite{saade2016lighton}) the evaluation of the complex exponential in~\eqref{eq:basicSketch} is complicated to implement in hardware---the costliest step in fast sketch computation \cite{chatalic2018fastSketch}.

Inspired by recent works concerning 1-bit random embeddings \cite{Boufounos2013efficientCodingQuantized}, we propose a new sketch procedure, illustrated in Fig.~\ref{fig:cloud}. This one is conceptually much simpler to integrate directly in hardware (\eg using voltage controlled oscillators~\cite{yoon2008time}), bypassing the high-bitrate signal acquisition. We replace the costly $\exp(- \im \hspace{1px} \cdot)$ signature function by \emph{1-bit universal quantization} $\Quant(\cdot) = \sign(\cos( \cdot)) = 2(\floor{\frac{\cdot}{2 \pi}} \mod 2) - 1$. This function (represented in Fig.~\ref{fig:cloud}c), corresponds to taking the least significant bit (LSB) of a uniform quantizer with quantizer stepsize~$\pi$. We justify this modification by proving that the periodicity of the signature function is much more important than its particular shape. While our sketch is cheaper to compute, it retains the advantages of the original one, \ie it is linear (prone to distributed computing) and its required size $m$ scales---as we show in our experiments---still as $\cl O(nK)$, with only a $15$ to $25\%$ increase compared to \texttt{CKM}.

\paragraph{Outline:}
Sec.~\ref{sec:compressive-k-means} recapitulates how \texttt{CKM} performs clustering from the sketch $\Vec{z}_X$. We propose a generalized sketch $\Vec{z}_{X,f}$ in Sec.~\ref{sec:generalizedSketch} where the signature $\exp(-\im \cdot)$ is replaced by a generic periodic function $f(\cdot)$. Our main result states that, in the \texttt{CKM} clustering method, $\Vec{z}_{X,f}$ can be used instead of $\Vec{z}_{X}$, even though $f$ is potentially non-differentiable, thanks to the addition of a random \textit{dithering} on the argument of $f$. This claim is supported by the possibility to recover the cost function implicitly minimized in \texttt{CKM} from $\Vec{z}_{X,f}$ (Prop.~\ref{prop:main}). Based on this observation, in Sec.~\ref{sec:quantized-ckm} we define our Quantized Compressive K-Means (\texttt{QCKM}) method that solves KMC from $\Vec{z}_{X,\Quant}$, a 1-bit sketch of the dataset associated with $f=\Quant$. We validate experimentally that \texttt{QCKM} competes favorably with \texttt{CKM} in Sec.~\ref{sec:experiments}, before concluding in Sec.~\ref{sec:conclusion}.

\paragraph{Related work:}
Most fast clustering methods for massive datasets rely on sample-wise dimensionality reduction \cite{Boutsidis2009UnsupervisedFeatureSelectionKmeans, Boutsidis2010randomProjectionsCluster,Needell2017classificationBinary,dupraz2018binarykmeans}. One notable exception is the coresets method \cite{Frahling08} that proposes to subsample the dataset to both approximate the SSE and boost K-means. For the related kernel K-means problem, \cite{Chitta2012clusterRFF} uses Random Fourier Features \cite{Rahimi2008RFF}, \ie the low-dimensional mapping $\Vec{z}_{(\cdot)}$ defined in~\eqref{eq:basicSketch}. For $\bs u,\bs v \in \bb R^n$, the inner product $\scp{\Vec{z}_{\Vec{u}}}{\Vec{z}_{\Vec{v}}}$ approximates a shift-invariant kernel $\kappa(\bs u, \bs v)$ associated with the frequency distribution $\Lambda$. \texttt{CKM} \cite{keriven2016compressive} actually \textit{averages} individual RFF of data points. 
Interestingly, $\kappa$ also defines a Reproducing Kernel Hilbert Space in which two probability density functions (pdf’s) can be compared with a Maximum Mean Discrepancy (MMD) metric~\cite{gretton2007mmd,smola2007hilbert,Sriperumbudur2011densityEstimationRKHS,Aronszajn1950rkhsThm,keriven2016GMMestimation}. Equipped with the MMD metric, the Generalized Method of Moments~\cite{Hall2005generalizedMethodMoments} in~\eqref{eq:CKMclustering} is equivalent to an infinite-dimensional Compressed Sensing \cite{Candes2008CSintro} problem, where the ``sparse'' pdf underlying the data (\eg approximated by few Diracs) is reconstructed from a small number of compressive, random linear pdf measurements: the sketch~\cite{keriven2016GMMestimation}. The method to solve~\eqref{eq:CKMclustering} is thus inspired by the \texttt{OMP(R)} CS recovery algorithm, \ie Orthogonal Matching Pursuit (with Replacement) \cite{Pati1993OMP,Jain2011OMPR}. In this work, analogously to how RFF are generalized to any periodic signature of random projections in \cite{Boufounos2013efficientCodingQuantized,Boufounos2015signalGeometry}, \textit{pooled} RFF (the dataset sketch) are generalized to \textit{pooled} periodic signatures of random projections, with universal quantization (known to preserve local signal distances) as a particular case. 

\section{Background: Compressive K-Means (CKM)}
\label{sec:compressive-k-means}

Most unsupervised learning tasks amount to estimating (some parameters of) the unknown probability distribution $\cl P$ from $N$ learning examples $\Vec{x}_i \distiid \cl P$. Compressive learning aims at estimating $\cl P$ from its sketch $\cl A (\cl P) \in \bb C^m$: a random sampling of its characteristic function $\phi_{\cl P}(\Vec{\omega}) := \expec{\Vec{x} \sim \cl P} e^{\im \Vec{\omega}^T\Vec{x}}$ at $m$ frequencies $\Omega = (\Vec{\omega}_1,\cdots,\Vec{\omega}_m)$ drawn from a well-specified distribution $\Vec{\omega}_j \distiid \Lambda$.
The \textit{sketch operator} reads\footnote{Scalar functions (\eg $\exp$) are here applied component-wise on vectors.}
\begin{equation}
\label{eq:usualSketch}
 \cl A(\mathcal P) := \expec{\Vec{x} \sim \mathcal{P}} e^{- \im \Omega^T \Vec{x} } = (\phi^*_{\cl P}(\bs \omega_j))_{j=1}^m
 \quad \simeq \quad \Vec{z}_X := \cl A ( \hat{\cl P}_X ) =  \frac{1}{N} \sum_{\Vec{x}_i \in X} e^{- \im \Omega^T \Vec{x}_i }, 
\end{equation}
\noindent where the sketch $\Vec{z}_X$ of a dataset $X$ actually refers to the sketch of its empirical pdf, $\hat{\cl P}_X := \frac{1}{N} \sum_{\Vec{x}_i \in X} \delta_{\Vec{x}_i}$, as announced in~\eqref{eq:basicSketch}. Given its sketch $\cl A (\cl P)$, $\cl P$ can be approximated by a pdf $\mathcal{Q}$ belonging to some simple (``sparse") model set $\cl G$---where the approximation error is quantified by the MMD metric, that can in this particular context be written as $\gamma^2_{\Lambda}(\cl P, \cl Q) :=  \expec{\Vec{\omega} \sim \Lambda} | \phi_{\cl P}(\Vec{\omega}) - \phi_{\cl Q}(\Vec{\omega}) |^2$ \cite{Sriperumbudur2011densityEstimationRKHS,Sriperumbudur2010hilbertEmbedding}. The $\ell_2$ sketch distance serves as an estimate for $\gamma_{\Lambda}(\cl P, \cl Q)$, hence in practice $\cl Q$ is found by solving the \textit{sketch matching problem}:
\begin{equation}
\label{eq:sketchmatching}
\mathcal{Q}^* \in \argmin_{\mathcal{Q} \in \mathcal{G}} \: \| \cl A ( \mathcal{P} ) - \cl A ( \mathcal{Q} ) \|^2 \:\: \simeq \:\: \argmin_{\mathcal{Q} \in \mathcal{G}} \: \gamma_{\Lambda}^2(\cl P, \cl Q). 
\end{equation}
In \texttt{CKM}, $\cl A(\cl P)$ is approximated by $\cl A(\hat{\cl P}_X)$, and $\cl Q$ is a weighted mixture of $K$ Diracs located at the the centroids $\Vec{c}_k \in \cl C \subset \Rbb^n$, \ie $\cl G := \{ \sum_{k = 1}^{K} \alpha_k \delta_{\Vec{c}_k} : \Vec{c}_k\in \cl C,\, \alpha_k \geq 0,\, \sum \alpha_k = 1\}$. From~\eqref{eq:sketchmatching} the \texttt{CKM} objective function reads: 
\begin{equation}
\label{eq:ckm}
(\cl C_{\texttt{CKM}},\Vec{\alpha}_{\texttt{CKM}}) \in \arg  \min_{(\cl C, \bs \alpha)} \: \norm{ \Vec{z}_X - \cl A ( \ts \sum_{k=1}^{K} \alpha_k \delta_{\Vec{c}_k} )} ^2, 
\end{equation}
as announced in~\eqref{eq:CKMclustering}. This non-convex problem is hard to solve exactly, but the \texttt{CKM} algorithm (based on \texttt{OMPR}) \cite{keriven2016compressive}, detailed in pseudocode below, seeks an approximate solution. More precisely, \texttt{CKM} greedily selects new centroids minimizing a residual $\Vec{r} \in \bb C^m$ (Steps 1 and 2) inside a box with lower and upper bounds $\Vec{l}, \Vec{u} \in \bb R^n$, respectively, enclosing the data $X$, and eventually replacing bad centroids in Step 3. 
The centroid weights $\alpha_k$ are then computed and a global gradient descent initialized at the current values allows further decrease of the objective (Steps 4 and 5). \texttt{CKM} relies on solving several (not always convex) optimization sub-problems, in practice solved approximately (a \textit{local} optimum is found) using a quasi-Newton optimization scheme. 

\begin{algorithm}
	\SetKwData{Left}{left}\SetKwData{This}{this}\SetKwData{Up}{up}
	\SetKwFunction{Union}{Union}\SetKwFunction{FindCompress}{FindCompress}
	\SetKwInOut{Input}{input}\SetKwInOut{Output}{output}
	$\Vec{r} \leftarrow \Vec{z}_X$, $\cl C \leftarrow \emptyset$ \emph{(Initialize residual and  centroids)} \\
	\For{$t = 1,\, \cdots, 2K$} 
	{ 
		\textbf{Step 1} : gradient descent selects $\Vec{c}$ highly correlated with residual:
		
		$\Vec{c} = \texttt{maximize}_{\bar{\Vec{c}}} \: \Re \langle \frac{\mathcal{A}\delta_{\bar{\Vec{c}}}}{\|\mathcal{A}\delta_{\bar{\Vec{c}}}\|},\Vec{r} \rangle \quad \text{s.t.} \quad  \Vec{l} \leq \bar{\Vec{c}} \leq \Vec{u}$ \\
		
		\textbf{Step 2} : add it to the support:
		
		$\cl C \leftarrow \cl C \cup \{\Vec{c}\}$\\
		
		\textbf{Step 3} : Reduce support by Hard Thresholding:
		
		\If{$|\cl C| > K$}{
			$\Vec{\beta} = \arg \min_{\bar{\Vec{\beta}}} \: \norm[\big]{\Vec{z}_X - \sum_{k=1}^{|\cl C|} \bar{\beta}_k \frac{\mathcal{A}\delta_{\Vec{c}_k}}{\|\mathcal{A}\delta_{\Vec{c}_k}\|}}\quad \text{s.t.}\quad \bar{\Vec{\beta}} \in \mathbb{R}^{|\cl C|}_+$ \\
			$\cl C \leftarrow$ set of $K$ centroids $\Vec{c}_k$ corresponding to $K$ largest magnitude values of $\Vec{\beta}$.
		}
		
		\textbf{Step 4} : Project to find optimal weights:\vspace{1mm}
		
		$\Vec{\alpha} = \arg \min_{\bar{\Vec{\alpha}}} \: \norm[\big]{\Vec{z}_X - \sum_{k=1}^{|\cl C|} \bar{\alpha}_k \mathcal{A}\delta_{\Vec{c}_k}}\quad\text{s.t.}\quad \bar{\Vec{\alpha}} \in \mathbb{R}^{|\cl C|}_+$ \\
		
		\textbf{Step 5} : Global gradient descent\vspace{1mm}
		
		$(\cl C,\Vec{\alpha}) \leftarrow \texttt{minimize}_{\bar{\cl C}=\{\bar{\bs c}_k\}, \bar{\Vec{\alpha}}} \: \norm[\big]{\Vec{z}_X - \sum_{k} \bar{\alpha}_k \mathcal{A}\delta_{\bar{\Vec{c}}_k}} \: \ \text{s.t.}\  \Vec{l} \leq \bar{\Vec{c}}_k \leq \Vec{u}$
		
		$\Vec{r} = \Vec{z}_X - \sum_{k=1}^{|\cl C|} \alpha_k \mathcal{A}\delta_{\Vec{c}_k}$ \emph{(Update residual)}\\
	}
	\caption*{\texttt{CKM}: Compressive $K$-Means clustering.}\label{algo:ckm}
\end{algorithm}

\paragraph{CKM parameters:} The frequency distribution $\Lambda(\Vec{\omega})$ ought to define a meaningful metric $\gamma_{\Lambda}$, \ie the objective function of \texttt{CKM}. By Bochner's theorem \cite{Rudin1962bochnerBook}, $\Lambda$ is associated with a positive definite, translation-invariant kernel (``similarity measure") $\kappa(\Vec{x},\Vec{x}') = K(\Vec{x}-\Vec{x}')$ through the Fourier transform $K(\Vec{u}) = \cl F(\Lambda)(\Vec{u}) := \Integr{}{}{e^{- \im \Vec{u}^T \Vec{\omega}}}{\Lambda(\Vec{\omega})}$. Concretely, $\Lambda$ limits the frequencies of $\cl P$ we are able to observe in $\cl A(\cl P)$, acting as a ``low-pass filter" convolving $\cl P$ with $K$. $\Lambda$ thus implicitly controls \texttt{CKM}'s clustering scale, and requires some \textit{a priori} insight about~$\cl P$. In practice \texttt{CKM} uses heuristics adjusting $\Lambda$ from a subset of $X$ \cite{keriven2016GMMestimation}. For the required sketch dimension $m$, \cite{gribonval2017compressiveStatisticalLearning} provides theoretical guarantees when $m = \cl O(nK^2)$ but experiments strongly suggest $\cl O(nK)$ is sufficient~\cite{keriven2016compressive}. 

\section{Sketching with general signature functions}
\label{sec:generalizedSketch}

We here generalize the sensing function $\exp(-\im \cdot)$ of the sketch~\eqref{eq:usualSketch} to a general (\eg discontinuous) periodic function~$f(t)$ assumed (w.l.o.g.) $2\pi$-periodic, centered and taking values in $[-1,1]$. Therefore $f(t) = \sum_{k \neq 0} F_k e^{\im k t}$, with Fourier series coefficients $F_k$ such that $F_0=0$, and $F_{\pm 1} \neq 0$ (up to a rescaling of $f$). The \textit{generalized sketch operator} $\cl A_f$ is 
\begin{equation}
\ts \cl A_f (\cl P) := \expec{\Vec{x} \sim \cl P} f(\Omega^T\Vec{x} + \Vec{\xi}) \quad \text{for} \quad \Vec{\omega}_j \distiid \Lambda(\Vec{\omega}), 
\end{equation}
with a uniform \emph{dithering} $\xi_j \distiid \cl U([0,2 \pi])$. 

Our main question is now: \emph{given $\cl A_f (\cl P)$, is it still possible to approximate $\cl P$ by some low-complexity distribution $\cl Q$, as done in~\eqref{eq:sketchmatching}?} Our answer is positive since we can 
still approximate the same objective function, \ie the MMD metric $\gamma_{\Lambda}(\cl P, \cl Q)$, from this new sketch. Intuitively, the dithering $\Vec{\xi}$ allows us to ``separate" $\cl A_f (\cl P)$ into two terms: one associated with the low frequencies of $f$, that contributes to the target objective $\gamma_{\Lambda}(\cl P, \cl Q)$, and one ``high-frequency" term that is constant for the relevant optimization problem. This is formally proven in the following proposition. 
\begin{proposition}
	\label{prop:main}
	Given pdfs $\cl P, \cl Q$, denoting $f$'s first harmonic as $f_1(t) := \sum_{k \in \{\pm 1\}} F_k e^{\im t} $, there is a constant $c_{\cl P} > 0$ such that
	\begin{equation}
	\label{eq:accessAllFreq}
	\ts \big| (2 m | F_1 |^2)^{-1} \norm{ \mathcal A_f(\mathcal P) - \mathcal{A}_{f_1} (\mathcal{Q}) }^2  -  \gamma^2_{\Lambda}(\mathcal P, \mathcal Q) -  c_{\mathcal{P}} \big| \leq \epsilon, 
	\end{equation}
	with probability exceeding $1 - 2e^{- C_f m \epsilon^2}$on the draw of $\Omega$ and~$\Vec{\xi}$, for $C_f = 8|F_1|^4\,(1 + 2|F_1|)^{-4}$. 
\end{proposition}
\begin{proof}
	Note that 
	$\|\mathcal A_f(\mathcal P) - \mathcal{A}_{f_1} (\mathcal{Q})\|^2 = \sum_{j=1}^m Z_j$ with $Z_j := |\expec{\Vec{x} \sim \mathcal{P}} f(\Vec{\omega}_j^T\Vec{x} + \xi_j) - \expec{\Vec{y} \sim \mathcal{Q}} f_1(\Vec{\omega}_j^T\Vec{y} + \xi_j)|^2 \leq  (1+2|F_1|)^2 $. Define ${\tilde{F}_k = F_k (\delta_{k',1} + \delta_{k',-1})}$ with $\delta_{p,q}$ the Kronecker delta. Since $\expec{\xi} e^{\im k' \xi} = \delta_{k',0}$, $\mu_Z := \expec{\Vec{\omega}_j,\xi_j} Z_j$ reads 
	\begin{equation*}
	\begin{split} 
	\mu_Z &= \expec{\Vec{\omega}} \expec{\xi} \left|  \sum_{k \neq 0} e^{\im k \xi} ( F_k   \expec{\Vec{x} \sim \mathcal{P}} e^{\im k \Vec{\omega}^T\Vec{x}}     -  \tilde{F}_k   \expec{\Vec{x}' \sim \mathcal{Q}} e^{\im k \Vec{\omega}^T\Vec{x}'} ) \right|^2 \\[-1mm]
	& = \expec{\Vec{\omega} \sim \Lambda}   \sum_{k \neq 0}   | F_k   \phi_{\cl P}(k \Vec{\omega})     -  \tilde{F}_k   \phi_{\cl Q}(k \Vec{\omega})|^2 \\
	& = 2 | F_1 |^2 \gamma^2_{\Lambda}(\mathcal P, \mathcal Q) +  \expec{\Vec{\omega} }  \sum_{|k| \geq 2} |F_k|^2   |\phi_{\mathcal{P}}(k \Vec{\omega})|^2. 
	\end{split}
	\end{equation*}
	Hoeffding's inequality applied on the bounded $Z_j$s gives~\eqref{eq:accessAllFreq} with $c_{\mathcal{P}} =   \expec{\Vec{\omega} }  \sum_{|k| \geq 2} \frac{|F_k|^2}{2 |F_1|^2}   |\phi_{\mathcal{P}}(k \Vec{\omega})|^2$ constant with respect to~$\mathcal{Q}$. Note that $c_{\cl P} < \infty$ since $f$ is bounded and $F_1 \neq 0$.
\end{proof}
Prop. \ref{prop:main} allows us---at the price of adding a dithering---to sketch a dataset into $\Vec{z}_{X,f} = \cl A_f(\hat{\cl P}_X)$ with a very large class of functions $f$, \eg that model a realistic sensing scheme. Indeed, it shows that, for a fixed pair of pdfs, replacing \texttt{CKM}s objective in~\eqref{eq:sketchmatching} by $\norm{ \mathcal A_f(\mathcal P) - \mathcal{A}_{f_1} (\mathcal{Q}) }^2$ still approximates (up to a harmless constant $c_{\cl P}$) the MMD metric $\gamma_{\Lambda}^2(\cl P,\cl Q)$: an intuitively good cost function as justified by previous work. Moreover, thanks to the fact that $f_1$ is a cosine, the relevant gradients of this new cost function enjoy the same nice analytic expressions as \texttt{CKM}.
	Fixing the probability of \eqref{eq:accessAllFreq}, we also see that (again, for a fixed pair of distributions) the approximation error $\epsilon$ decays like $\cl O(1/\sqrt m)$ as $m$ increases. However, it is yet unclear how large $m$ must be to characterize the quality of the solution of \eqref{eq:QCKMclustering}. This would require Prop. \ref{prop:main} to hold for all $\cl Q$ in a ``low-dimensional'' set $\cl G \ni \cl P$ as done in \cite{gribonval2017compressiveStatisticalLearning}, a generalization that we postpone to a future work.

\section{Quantized Compressive K-Means (QCKM)}
\label{sec:quantized-ckm}

We now instantiate the results of Sec.~\ref{sec:generalizedSketch} to the 1-bit universal quantization $f(t) = \Quant(t) = \sign(\cos(t)) \in \{-1,+1\}$ to construct a hardware-friendly \textit{quantized sketch}. The 1-bit universal quantizer $\Quant$ is a square wave and can be seen as the Least Significant Bit of a uniform quantizer with quantization stepsize $\pi$~\cite{Boufounos2013efficientCodingQuantized,Boufounos2015signalGeometry}. The resulting sketch operator $\cl A_{\Quant}$ on a pdf $\cl P$ (resp. the sketch $\Vec{z}_{X,\Quant}$ of a dataset $X$) is
\begin{equation}
\cl A_{\Quant}(\mathcal P) := \expec{\Vec{x} \sim \mathcal{P}} \Quant( \Omega^T \Vec{x} + \Vec{\xi} ) \quad \simeq \quad \Vec{z}_{X,\Quant} := \cl A_{\Quant}( \hat{\cl P}_X ) = \frac{1}{N} \sum_{\Vec{x}_i \in X} q(\Omega^T\Vec{x}_i + \Vec{\xi}) ,
\end{equation}
and the clustering problem in~\eqref{eq:ckm} is now replaced by
\begin{equation}
(\cl C_{\texttt{QCKM}},\Vec{\alpha}_{\texttt{QCKM}}) \in  \argmin_{(\cl C,\Vec{\alpha})} \| \Vec{z}_{X, \Quant} - \cl A_{\Quant_1} ( \sum_{k} \alpha_k \delta_{\Vec{c}_k} ) \|^2. 
\label{eq:QCKMclustering}
\end{equation}
where $\Quant_1$ denotes the first harmonic of $\Quant$ (a cosine). Interestingly, the contribution $\Vec{z}_{\Vec{x}_i, \Quant} = \cl A_{\Quant}(\delta_{\Vec{x}_i}) \in \{-1,1\}^m$ of each signal $\Vec{x}_i$ can be encoded by only $m$ bits, as illustrated Fig.~\ref{fig:cloud}. 

To solve~\eqref{eq:QCKMclustering}, we adapt the \texttt{CKM} algorithm to account for the changes in objective function, which we call the Quantized Compressive K-Means (\texttt{QCKM}) algorithm. More precisely, $\Vec{z}_{X}$ is replaced by $\Vec{z}_{X,\Quant}$ at initialization and in Steps 3, 4 and 5, and $\cl A \delta_{\Vec{c}}$ is replaced by $\cl A_{\Quant_1} \delta_{\Vec{c}}$ in Steps 1, 3, 4 and 5. Small modifications also take into account the addition of $\Vec{\xi}$. 

\section{Experiments to validate QCKM}
\label{sec:experiments}

We show now empirically that \texttt{QCKM} requires only $m =\cl O(nK)$ measurements to find good centroids, with a hidden multiplicative constant only slightly higher ($15$ to $25\%$) than for \texttt{CKM}---remembering that \texttt{QCKM} receives $m$-bit sketch contributions whereas \texttt{CKM} uses full-precision contributions. We validate \texttt{QCKM} on both synthetic and real datasets and compare the performance with \texttt{k-means} (built-in MATLAB function) as well as \texttt{CKM} (from the SketchMLbox toolbox~\cite{sketchMLtoolbox}). 

\paragraph{Synthetic data:} We compute \emph{phase transition diagrams} (Fig.~\ref{fig:phaseN}) to highlight the relationship between the required amount of measurements $m$, and the sample dimension $n$ or the number of clusters $K$. For this, we arbitrary say that \texttt{(Q)CKM} is \textit{successful} if $\mathrm{SSE}_{\texttt{(Q)CKM}} \leq 1.2 \hspace{1.2px} \mathrm{SSE}_{\texttt{k-means}}$, where $\mathrm{SSE}_{\texttt{k-means}}$ is the best out of 5 \texttt{k-means} runs. These diagrams show how the empirical success rate (averaged over $100$ trials) of \texttt{QCKM} evolves with $m$, as $n$ or $K$ varies. For fair comparison with the complex exponential sketch (composed of a cosine and sine in its real and imaginary part, respectively), the $j^{\rm th}$ measurement of the quantized sketch is, in our experiment, composed of two measurements with the same frequency $\Vec{\omega}_j$ but two dithering values $\xi_j$ and $\xi_j + \frac{\pi}{2}$.
First, we draw $N=10000$ samples uniformly from $K = 2$ isotropic Gaussians in \textit{varying dimension} $n$, with means $\pm (1,\cdots,1)^\top \in \bb R^n$ and covariance matrix $\frac{n}{20} {\rm Id}$. The phase transition diagram is reported Fig.~\ref{fig:phaseN}a, along with lines showing the transition to more than $50\%$ success rate of \texttt{QCKM} (red solid) and, for comparison, of \texttt{CKM} (yellow dotted). This transition happens (except for a deviation at small dimensions) at a constant value of $m/nK$: as \texttt{CKM}, \texttt{QCKM} requires $m$ to be proportional to $n$. In this experiment, \texttt{QCKM} requires about 1.13 more measurements than \texttt{CKM}s (complex and full precision) measurements. Fig.~\ref{fig:phaseK}b is the phase transition for \textit{varying numbers of centroids} $K$ while fixing $n = 5$. Samples are drawn from $K$ Gaussians with means chosen randomly in $\{ \pm 1\}^n$, other parameters being identical to the previous experiment. Successful estimation occurs when $m$ scales linearly with $K$, with a factor of about $1.23$ between \texttt{QCKM} and \texttt{CKM} sample complexities.
These experiments suggest that \texttt{CKM}s empirical rule $m = \cl O(nK)$ holds for \texttt{QCKM}, with a slightly higher multiplicative constant.

\begin{figure}
	\centering
	\includegraphics[width=0.70\linewidth]{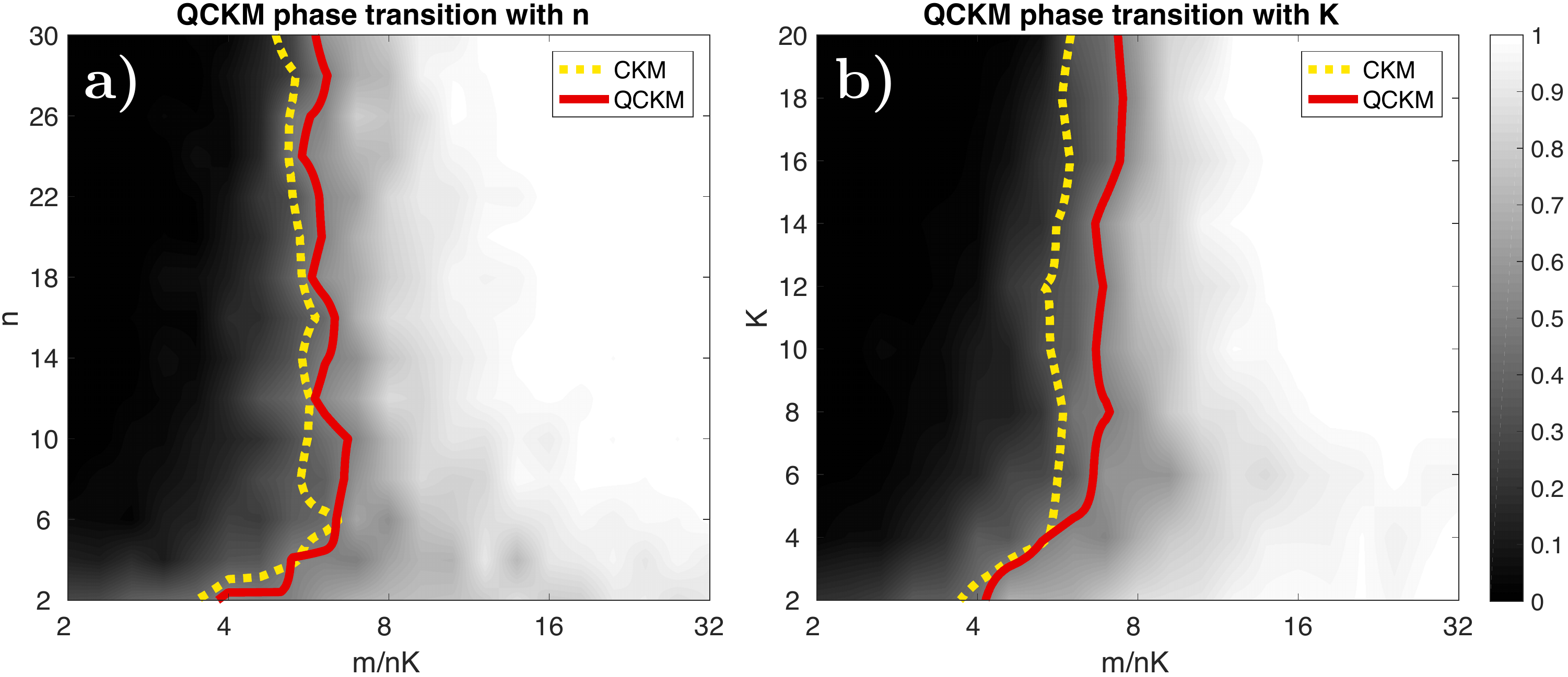}
	\caption{Empirical success rate---from $0\%$ (black) to $100\%$ (white)---evolution of \texttt{QCKM} with $m/nK$, and \textbf{(a)} $n$, or \textbf{(b)} $K$. The red solid line shows the transition to a success rate above $50\%$ for \texttt{QCKM}; for comparison, the yellow dotted line shows the same transition for \texttt{CKM}.}
	\label{fig:phaseN}
	\label{fig:phaseK}
\end{figure}

\paragraph{Real datasets:} The performance of \texttt{QCKM} are also assessed on real (and non-Gaussian) data: the spectral clustering (SC) \cite{vonLuxburg2007tutorialSpectralScluster} of the MNIST dataset (70000 $28\times28$ pixel images of handwritten digits~\cite{MNISTdataset}). This experiment aims at detecting, in an unsupervised setting, the 10 clusters corresponding to the digits $0-9$ from their representation in a 10-dimensional feature space\footnote{We thank the authors of \cite{keriven2016compressive} for having shared this SC dataset.}. We run the compressive clustering algorithms with $m = 1000$ frequencies. To avoid bad local minima, several replicates of \texttt{k-means} are usually run and the solution that achieves the best SSE is then selected. We thus also perform several replicates of \texttt{(Q)CKM}, but since computing the SSE requires access to whole dataset (which is not supposed available to the compressive algorithms), we select the solution of \texttt{CKM} (resp. \texttt{QCKM}) minimizing~\eqref{eq:ckm} (resp.~\eqref{eq:QCKMclustering}) \cite{keriven2016compressive}. 

We use two performance metrics to assess the clustering algorithms: the SSE in~\eqref{eq:clustering}, an obvious KMC quality measure, and the Adjusted Rand Index (ARI) \cite{vinh2010informationARI} that compares the clusters produced by the different algorithms with the ground truth digits. A higher ARI means the clusters are closer to the ground truth, with ARI $= 1$ if the partitions are identical and ARI $= 0$ (on average) if the clusters are assigned at random.

\begin{figure}
	\centering
	\includegraphics[width=0.70\linewidth]{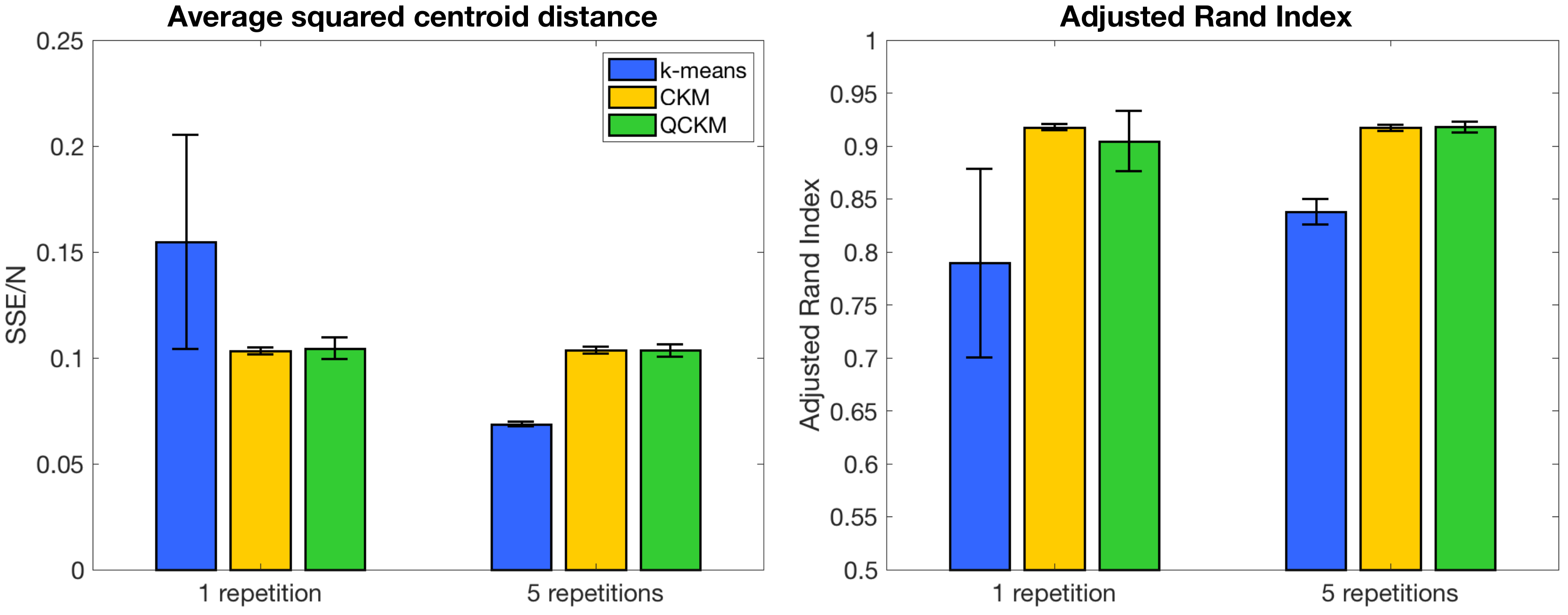}
	\caption{Mean with standard deviation over 100 experiments of the performance ($SSE/N$ and ARI on left and right, respectively) of the different compared clustering algorithms (\texttt{k-means} in blue, \texttt{CKM} in yellow, \texttt{QCKM} in green), both for 1 and 5 repetitions of the learning algorithms.}
	\label{fig:MNIST}
\end{figure}

Fig.~\ref{fig:MNIST} reports the mean and standard deviation (excluding a few clear outliers for \texttt{CKM} and \texttt{QCKM}, occuring about $5\%$ of the time on average) obtained for both performance metrics. Globally, \texttt{QCKM} performs similarly to \texttt{CKM}, retaining its advantages over \texttt{k-means}. First, the compressive learning algorithms are more stable: their performance exhibit small variance, in contrast with \texttt{k-means} that hence benefits the most from several replicates. In addition, while for several replicates \texttt{k-means} outperforms \texttt{(Q)CKM} in terms of SSE, the solutions of \texttt{(Q)CKM} are closer (as quantified by the ARI score) to the ground truth labels: this suggests that the objectives~\eqref{eq:ckm} or~\eqref{eq:QCKMclustering} are better suited than the SSE for (at least) this task. Note that \texttt{QCKM} performances have moderately higher variance than those of \texttt{CKM}: this is probably due to the
increased measurement rate of \texttt{QCKM} required to reach similar performance (as suggested by the first experiments) while here both algorithms ran with $m = 1000$. 

\section{Conclusion}
\label{sec:conclusion}

In the context of compressive learning, we have shown that replacing the complex exponential by any periodic function $f$ in the sketch procedure can be compensated by \textit{i)} adding a dithering term to its input, and \textit{ii)} retaining only the first harmonic of $f$ at reconstruction. However, Prop.~\ref{prop:main} is valid for \textit{fixed} distributions: future work should provide guarantees for \textit{all} distributions, \eg belonging to some low-complexity set $\cl G$. Still, we believe this result can simplify the design of \textit{low-power sensors acquiring only the minimal information required for some learning task} (\eg the sketch contribution). To illustrate this idea, we proposed \texttt{QCKM}, a compressive clustering method based on \texttt{CKM} \cite{keriven2016compressive} but using hardware-friendly, 1-bit sketches of the learning data, and validated this approach through experiments. However, while yielding promising results in practice, the greedy algorithms that try to solve the non-convex sketch matching optimization problem (\eg \texttt{CKM} and \texttt{QCKM}) still lack theoretical convergence guarantees. Future work could also consider the binarization of new learning tasks (\eg Compressive PCA \cite{gribonval2017compressiveStatisticalLearning}), or explore other sketching mechanisms.

\end{document}